\documentclass[conference]{IEEEtran}
\usepackage{times}

\usepackage[numbers]{natbib}
\usepackage{multicol}
\usepackage[bookmarks=true]{hyperref}

\usepackage{xcolor}
\usepackage{color}

\newcommand{\nn}[1]{\textcolor{black}{#1}}
\usepackage[normalem]{ulem} 

\usepackage{float}
\usepackage{tikz}
\usetikzlibrary{arrows.meta,positioning,calc}
\definecolor{skyblue}{RGB}{135,206,235} 
\usepackage{subcaption} 

\usepackage{tabularx}

\usepackage{algorithm}
\usepackage{algpseudocode}
\algrenewcommand\algorithmicrequire{\textbf{Input:}}
\algrenewcommand\algorithmicensure{\textbf{Output:}}

\usepackage{comment}
\usepackage{amsfonts}
\usepackage{amsmath}

\newtheorem{theorem}{Theorem}
\newtheorem{lemma}[theorem]{Lemma}

\begin{document}

\title{Optimal UGV–UAV Cooperative Partitioning and Inspection of Shortest Paths}

\IEEEoverridecommandlockouts


\author{
\authorblockN{Ninh Nguyen and Srinivas Akella}
\authorblockA{
Department of Computer Science\\
University of North Carolina at Charlotte,
Charlotte, NC 28223, USA\\
Email: {\tt\small \{tnguy248, sakella\}@charlotte.edu}
}
}



%

\maketitle

\begin{abstract}
We study cooperative shortest path planning for an unmanned ground vehicle (UGV) assisted by an unmanned aerial vehicle (UAV) scout in environments with unknown road blockages that are only discovered when a robot reaches the damaged point. This formulation generalizes the original Canadian Traveller Problem (CTP), which assumes a single ground vehicle and that the traversability status of all incident edges is revealed upon arrival at a vertex. We first analyze the case where the start and the goal are connected by $k$ disjoint paths, and prove that the worst-case competitive ratio $\rho$ for a single UGV is $2k-1$. With UAV assistance, and under the simplifying assumption of negligible initial transit and deadheading UAV costs, the ratio $\rho$ improves to $2(k-1)\frac{v_G}{v_G + v_A} + 1$, where $v_G$ and $v_A$ denote the UGV and UAV speed, respectively. To address general graphs and non-negligible UAV initial transit and deadheading costs, we present an optimal candidate-path partitioning algorithm that assigns path prefix inspection to the UGV and path suffix inspection to the UAV, and prove the optimality of the UAV inspection strategy on general graphs. We evaluate our algorithm by performing experiments on road networks from the world's 50 most populous cities with randomized blockage locations,
and show that the algorithm reduces UGV travel time, with larger improvements as the UAV speed increases.
\end{abstract}

\IEEEpeerreviewmaketitle

\section{Introduction}

Unmanned Ground Vehicles (UGVs) and Unmanned Aerial Vehicles (UAVs) are becoming essential tools in emergency response, logistics, and disaster relief~\cite{MunasinghePD24_UAVUGVReview, KashyapGMSD19, ChoudhurySKP22_arxiv}. In assessment and response scenarios after disasters such as earthquakes, floods, or wildfires, ground vehicles are tasked with transporting critical supplies or rescuing victims, while aerial robots perform reconnaissance and provide situational awareness. A central challenge in these environments is path planning under uncertainty; road conditions may be unknown in advance, and obstacles are often revealed only when a vehicle arrives at their location.  Guaranteeing safe and efficient navigation in such uncertain environments remains a pressing problem.

This challenge is closely related to the Canadian Traveller Problem (CTP)~\cite{PapadimitriouY91_BasicCTP}, where an agent must travel from a start vertex to a goal vertex in a graph where some edges may be blocked, and blockages are discovered only upon arrival. The stochastic CTP, where each edge is assumed to have a known probability of being blocked, has been widely studied~\cite{BarNoyS91_kCTP, NikolovaK08_StochasticCTP, BnayaFFMS21_RepeatedCTP, BnayaFS09_RemoteSensing}. In practice, however, prior knowledge of edge probabilities is rarely available or reliable in post-disaster environments. For the original CTP without probabilistic information, expected-value minimization is not applicable. Instead, researchers analyze competitive ratios~\cite{XuHSZZ09_CompetitiveAnalysis,Westphal08_kCTPAnalysis, BampisEX22_CTPPredictions} (the worst-case ratio between the cost of an online strategy and the shortest feasible path). Despite its practical importance, this problem has received comparatively less attention. 

In this paper, we study UGV–UAV cooperation for safe path planning in environments modeled similarly to the original CTP. However, the original CTP formulation makes two assumptions: edge status is revealed as soon as the traveller reaches a vertex incident to that edge, and the traveller is a single ground vehicle acting alone. In contrast, we consider a more realistic setting in which blockages are discovered only when a robot arrives at the damaged point of an edge, and path inspection can be shared between a UGV and a UAV. Since even the original CTP is computationally challenging, obtaining formal performance guarantees for UGV-UAV variants requires simplifying assumptions that make the analysis tractable.

Our theoretical analysis considers a simplified setting in which a start vertex and a goal vertex are connected by multiple disjoint paths. This isolates the difficulty caused by unknown blockages while enabling a competitive ratio analysis. In the single-agent case, an oblivious adversary may block all but the longest path near the goal, forcing the UGV to backtrack repeatedly. We prove that in this case the competitive ratio $\rho$ is $2k-1$, where $k$ is the number of disjoint paths. We  then introduce a cooperative strategy in which a UAV scout inspects potential paths as the UGV attempts traversal. Under the simplifying assumption that the initial transit cost for the UAV to reach a path before inspection, as well as the deadheading cost (the additional time for the UAV to travel between vertices of the path without performing inspection), are both small, we show that the competitive ratio $\rho$ improves to $2(k-1)\frac{v_G}{v_G + v_A} + 1,$ where $v_G$ and $v_A$ are the speeds of the UGV and UAV, respectively.

Motivated by the above theoretical result, we address the following question on general graphs: given a candidate shortest path selected by the UGV in a general graph, how should inspection of this path be divided between the UGV and UAV?
We present an optimal candidate-path partitioning algorithm that assigns path prefix inspection to the UGV and path suffix inspection to the UAV.   In realistic scenarios, both the UAV's initial transit cost and deadheading cost for the inspection plan must be considered. 
Hence, we present an efficient algorithm that guarantees a minimum travel time for UAV inspection, including the initial transit and deadheading costs, on general graphs.


The main contributions of this work are:
\begin{enumerate}

     \item  \textbf{Worst-case analysis of UGV-only and UGV--UAV path planning:} We analyze a simplified setting in which the start and goal are connected by \(k\) disjoint paths. For the UGV-only strategy, we provide a tight competitive-ratio bound of $2k-1$. We then show that UAV assistance improves the worst-case bound, with the improved ratio of $2(k-1)\frac{v_G}{v_G + v_A} + 1$ depending on the UGV and UAV speeds. 
    \item  \textbf{Optimal UGV-UAV inspection strategy:} Motivated by the theoretical result, we provide an optimal polynomial-time algorithm for  the subproblem of partitioning path inspection between the UGV and UAV in general graphs.
    \item \textbf{Simulation results:} We validate the optimal partition strategy on road networks from the 50 most populous cities with randomized blockage locations. 
    As the UGV:UAV speed ratio varies from 20:20 to 20:40, the cooperative strategy reduces UGV travel time relative to the UGV-only baseline by 
    \(14.30\%\) to \(17.97\%\) on average on large maps, with larger improvements as the UAV speed increases.
\end{enumerate}

The remainder of this paper is organized as follows:
Section~\ref{sec:related} reviews prior work on the Canadian Traveller Problem and cooperative multi-robot path planning under uncertainty.
Section~\ref {sec:formulation} presents our \nn{general UGV--UAV safe-path planning problem} formulation, including the graph model, edge status uncertainty, and subproblems.
\nn{Section~\ref{sec:analysis-ugv-only} introduces the simplified disjoint-paths model used for worst-case analysis and derives the UGV-only competitive ratio bound. Section~\ref{sec:analysis-ugv-uav} analyzes the cooperative UGV--UAV case under the same model.}
Section~\ref{sec:partition} introduces our optimal partitioning algorithm for UAV-assisted inspection in general graphs.
Section~\ref{sec:algorithms} describes the algorithm implementation with a complexity analysis.
Section~\ref{sec:results} reports simulation results on OpenStreetMap road networks with randomized blockages, and Section~\ref{sec:conclusion} concludes with a summary and directions for future work.

\section{Related Work}
\label{sec:related}


\subsection{Canadian Traveller Problem (CTP)}

The Canadian Traveller Problem (CTP), introduced by Papadimitriou and Yannakakis~\cite{PapadimitriouY91_BasicCTP}, models online path planning in graphs with unknown edge availability. In this problem, a traveler must move from a start vertex $s$ to a goal vertex $g$, where some edges may be blocked, and the status of each edge is revealed only when the traveler attempts to use it. For general graphs, it is PSPACE-complete to achieve a bounded competitive ratio for this problem, and the stochastic version has been shown to be \#P-hard.


The CTP can be viewed as a stochastic shortest path problem with recourse~\cite{PolychronopoulosT96_SSPPR}; recourse refers to the ability to adjust or modify a decision when new information is revealed.
Many CTP strategies repeatedly update the graph when blockages are discovered (e.g., by deleting blocked edges) and then recompute shortest paths from the traveler’s current location to the goal $g$. This connects practical CTP implementations to dynamic graph algorithms for shortest-path computation under edge updates~\cite{DemetrescuEGI10, HanauerHS22}.

Several important extensions of the original CTP have been studied~\cite{BarNoyS91_kCTP, BnayaFS09_RemoteSensing, FriedSBW13_CTPComplexity, Westphal08_kCTPAnalysis}. One of them is $k$-CTP~\cite{BarNoyS91_kCTP}, where no more than $k$ edges in the graph may be blocked. Westphal~\cite{Westphal08_kCTPAnalysis} showed that no deterministic online strategy can achieve a competitive ratio smaller than $2k+1$, and the lower bound for any randomized online algorithm is $k+1$. Later, Bender and Westphal~\cite{BenderW12_OptimalRandomkCTP} proved that even if all $s$-$g$ paths are node-disjoint paths, the lower bound is still the same. Fried et al.~\cite{FriedSBW13_CTPComplexity} provided a detailed analysis of the computational complexity of various CTP variants, consolidating hardness results across multiple formulations.
All the above studies of CTP restrict attention to a single ground vehicle. In contrast, our work explores extensions of the CTP involving a heterogeneous cooperating team of a UGV and a UAV, and introduces a more realistic sensor observation model where blockages are discovered only when a robot reaches the damaged point of an edge.

\subsection{UGV-UAV Cooperation}
UGV-UAV cooperation~\cite{MunasinghePD24_UAVUGVReview, KashyapGMSD19, ChoudhurySKP22_arxiv} has been explored in robotics for tasks such as target tracking, exploration, and search and rescue, but not in the context of the original CTP. Bhadoriya et al.~\cite{BhadoriyaRDCM24} studied assisted UGV-UAV path planning in stochastic networks, where the UAV inspects the most critical edges with candidates from the UGV’s $k$-shortest paths. Their model assumes travel costs are random variables, but that all edges remain passable. This may not reflect real-world scenarios where some edges are completely impassable.
CTP with remote sensing considers the use of a remote sensor to minimize the expected travel cost of the vehicle~\cite{BnayaFS09_RemoteSensing} but similarly relies on prior knowledge of edge-damage probabilities, which is often unavailable in post-disaster activities.
Our recent work~\cite{NguyenA26} studied the Dynamic UGV-UAV Cooperative Path Planning problem, in which one UGV is assisted by one or more UAVs in an uncertain road network with potentially impassible edges. It proposed and evaluated several UAV inspection strategies, including a bidirectional strategy, and showed that UAV inspection can reduce the UGV's travel time. However, it was primarily an empirical study and did not provide worst-case competitive ratio guarantees.



\subsection{Line Coverage and Inspection Problems}

When the UGV identifies a candidate shortest path to the goal, the UAV can assist by inspecting edges along that path. We need to plan how to traverse all required edges of a path efficiently. Such problems are related to arc routing and line coverage problems~\cite{CorberanL14_ArcRouting, AgarwalA24}, 
which studies how to cover all required edges in a graph with minimum cost.

While these formulations address efficient edge coverage, they do not consider heterogeneous cooperation or the partitioning of inspection tasks between multiple agents. Since our subproblems are specialized to the case where the required edges form a simple path, we show that inspection can be divided optimally between the UGV and UAV in polynomial time.

\section{Problem Formulation}
\label{sec:formulation}

The task is for the UGV to travel from start vertex $s$ to goal vertex $g$ in a graph $G$ with UAV assistance to minimize the UGV's total travel time under unknown edge traversability conditions. For a candidate shortest path $P_i=\langle u_0,\dots,u_m=g\rangle$ at planning step $i$, we reduce the problem into two subproblems:
\begin{enumerate}
    \item UGV-UAV Path Partitioning (Subproblem 1): Choose a vertex $u_j$ that partitions $P_i$ into a UGV prefix \nn{$P_i^{\text{UGV}}=\langle u_0,\dots,u_j \rangle$} and a UAV suffix \nn{$P_i^{\text{UAV}}= \langle u_j,\dots,u_m \rangle$}, with the UGV traversing the prefix while the UAV inspects the suffix. 
    \item UAV Inspection of a Subpath (Subproblem 2): Given $u_j$ and the suffix, compute an optimal UAV plan (including initial transit and necessary deadheading) that inspects all suffix edges in minimum time. 
\end{enumerate}


\subsection{Graph Model}
We represent the road network as a weighted undirected graph $G = (V, E, w)$, where $V$ is the set of vertices, $E \subseteq V \times V$ is the set of edges (i.e., road segments), and $w:E \to \mathbb{R}^+$ assigns each edge a positive length. A start vertex $s \in V$ and a goal vertex $g \in V$ are designated. The length of the shortest feasible path from $s$ to $g$ in $G$ is denoted by $L^*$.

\subsection{UGV and UAV Motion}
The UGV is restricted to move along edges of the graph $G$ at constant speed $v_G > 0$.
The UAV, on the other hand, flies in continuous 2D space at a constant altitude with a constant speed $v_A > 0$, and can deadhead (i.e., fly directly) between any pair of vertices, and inspect edges by flying over them.

\subsection{Edge Status and Inspection}
Each edge $e \in E$ may be either \textit{open} or \textit{blocked}. The status of an edge is unknown until it is inspected by either the UGV or the UAV. An edge can be blocked at any point along its length.

\paragraph{UGV traversal}
If the UGV is at vertex $u$ and attempts to traverse an edge $e=(u,v)$ that is blocked, it can only discover the blockage after reaching the obstruction. In this case, the UGV must return to $u$ before finding a new route. If $d$ is the distance from $u$ to the obstruction along $e$, the travel cost of this failed attempt is $2d$.

\paragraph{UAV inspection}
The UAV inspects an edge by flying over its full length, and detects any blockage during the flight. Whether the edge is blocked or open, we assume the UAV completes traversal of the edge.  Upon reaching the opposite vertex, the UAV communicates the edge status to the UGV, enabling the UGV to reroute without physically attempting traversal of a blocked edge.


The cooperative path planning problem for a UGV-UAV team in general graphs with blockages is challenging because the set of damaged edges and the positions of the damaged points are unknown in advance, and every inspection decision changes the subsequent positions of both robots. Our online algorithm therefore repeatedly computes a candidate shortest path to the goal using the currently known graph. For each candidate path, the UGV and UAV cooperate to inspect the path: the UGV is assigned a prefix of the path, while the UAV is assigned a suffix. This motivates the following two subproblems.

\subsection{Subproblem 1: UGV-UAV Path Partitioning}
\label{subsec:split-path}

Suppose that at planning step \(i\), a blockage has been discovered on the current UGV candidate path, either by the UGV or by the UAV. The discovered blocked edge is removed from the currently known graph, and the UGV replans from its current vertex. Let \(x_G^i\) denote the UGV's current vertex at this replanning step, and let $P_i = \langle u_0=x_G^i,u_1,u_2,\dots,u_m=g \rangle$ be the new candidate shortest path from \(x_G^i\) to the goal \(g\), with edge sequence \nn{$E_{P_i} = \{e_r\}_{r=1}^{m}$ with $e_r = (u_{r-1},u_r)$}.
In the disjoint-path worst-case analysis, the UGV returns to the start before attempting the next path and \(x_G^i=s\). In the general graph algorithm, however, \(x_G^i\) may be any vertex reached by the UGV during execution.


We partition the path $P_i$ between the UGV and UAV by selecting a split index $j \in \{0,\dots,m\}$ and assigning the \textit{path prefix} \nn{$P_i^{\text{UGV}}=\langle u_0,\dots,u_j \rangle$} to the UGV, and the \textit{path suffix} \nn{$P_i^{\text{UAV}}= \langle u_j,\dots,u_m \rangle$} to the UAV.

The UGV traverses and inspects the prefix of $P_i$, while the UAV inspects its suffix.  
Define the prefix length $ S(j) = \sum_{r=1}^j w(e_r)$, with $S(0)=0$.  
Let $\tau_G^i$ denote the time before the UGV can begin traversing the newly selected path \(P_i\), such as the time needed to return from a discovered damage point to the most recently visited vertex or to wait for the UAV to finish its current edge traversal and communicate inspection results.
Then the total UGV time to inspect its assigned prefix is $ T_{\text{UGV}}(j) = \tau_G^i + \frac{S(j)}{v_G}$. 
Let $T_{\text{UAV}}^{\star}(j)$ denote the minimum inspection time for the UAV to cover the suffix $P_i^{\text{UAV}}(j)$ (which is computed by solving subproblem 2).

To summarize, for path $P_i$, we seek the split that balances the two inspection tasks:
\[
j^\star \;=\; \arg\min_{j \in \{0,\dots,m\}} 
\max\{T_{\text{UGV}}(j),\,T_{\text{UAV}}^{\star}(j)\}.
\]

Figure~\ref{fig:uav-ugv-path}  shows an example of subproblem~1 when we need to find a vertex $u_j^{\star}$ to divide the path $P_i$, and assign the path prefix to the UGV and the path suffix to the UAV.

\begin{figure}[ht]
    \centering
    \begin{tikzpicture}[scale=0.8,>=stealth, node distance=1cm]
        \tikzstyle{vertex}=[circle,draw,fill=skyblue!40,minimum size=8mm,inner sep=0pt]
        
        \node[vertex] (UGV) at (2,4) {$x_G$};
        \node[draw=none, fill=none, font=\scriptsize,below=12pt] (UGV_label) at (2,4) {\textbf{UGV}};

        \node[vertex] (UAV) at (5,4) {$x_A$};
        \node[draw=none, fill=none, font=\scriptsize,below=12pt] (UAV_label) at (5,4) {\textbf{UAV}};
        
        \node[vertex] (u0)  at (0,3) {$u_0$};
        \node[vertex] (u1)  at (1.2,1.5) {$u_1$};
        \node[vertex] (uj)  at (2.7,1) {$u_j^*$};
        \node[vertex] (um2) at (4.2,1) {$u_{m-2}$};
        \node[vertex] (um1) at (6,2) {$u_{m-1}$};
        \node[vertex] (um)  at (7,3) {$u_m$};

        \draw[thick,black]  (UGV) to[out=150,in=100] (u0);
        \draw[thick,red] (u0) -- (u1);
        \draw[thick,red,dashed] (u1) -- (uj);

        \draw[thick,blue,dashed] (uj) -- (um2);
        \draw[thick,blue] (um2) -- (um1) -- (um);
        \draw[thick,black] (um) to[out=90,in=0] (UAV);

    \end{tikzpicture}
    \caption{Illustration of UAV and UGV cooperative inspection of a path $\langle u_0,...u_m\rangle$.
    The UGV path (red) and UAV path (blue) share a split point $u_j^*$. Dashed lines represent a path segment that may contain multiple nodes. }
    \label{fig:uav-ugv-path}
\end{figure}

\subsection{Subproblem 2: UAV Inspection of a Subpath}
\label{subsec:subprob2}

Once a split index $j$ is chosen, the UAV is assigned the suffix $ P_i^{\text{UAV}} = \langle u_j, u_{j+1}, \dots, u_m = g\rangle$, with edge sequence $E_{\mathcal{R}} = \{(u_{r-1},u_r)\}_{r=j+1}^{m}$.
Let $\ell = m-j$ and relabel this path suffix as a path \nn{ $ \mathcal{R} = \langle v_0 =u_j, v_1, \dots, v_{\ell}=u_m=g \rangle$}, with edges $E_{\mathcal{R}} = \{(v_{r-1}, v_r)\}_{r=1}^\ell$. $E_\mathcal{R}$ is the set of required edges that the UAV must inspect to reveal their status so that the UGV can continue or reroute its path.


\paragraph{Motion graph for the UAV}
Since the UAV can fly directly between any two vertices, its permitted motion is represented by a complete graph
$ G_A = (V_A, E_A), V_A = \{x_A, v_0, v_1, \dots, v_\ell\} $, where $x_A \in V$ is the UAV’s current position.
Each edge $(p,q) \in E_A$ has a weight
\[
w_A(p,q) = 
\begin{cases}
    w(p,q), & \text{if } (p,q) \in E, \\[6pt]
    d(p,q), & \text{if } (p,q) \notin E,
\end{cases}
\]
where $w(p,q)$ is the original road-network weight in $G$, and $d(p,q)$ is the Euclidean distance between $p$ and $q$. The time for the UAV to traverse $(p,q)$ is $w_A(p,q)/v_A$.

\paragraph{Required edges}
The UAV must inspect all edges of the path suffix $E_{\mathcal{R}}$, each of which must be flown over completely.

\paragraph{Cost components}
The UAV’s total inspection time consists of three parts:
\begin{enumerate}
    \item \textbf{Initial transit:} Transit time from current position $x_A$ to a chosen start vertex $v_q \in \{v_0,\dots,v_\ell\}$,
    \[
    T_{\mathrm{init}}(x_A,v_q) = \frac{w_A(x_A,v_q)}{v_A}.
    \]
    \item \textbf{Inspection cost:} Time to inspect every required edge in $E_\mathcal{R}$, giving
    \[
    T_{\mathrm{inspect}} = \frac{L_{\mathcal{R}}}{v_A}, 
    \quad where \quad L_{\mathcal{R}} = \sum_{r=1}^{\ell} w(v_{r-1},v_r).
    \]
    \item \textbf{Deadheading to inspect $\mathcal{R}$:} Transit time for additional traversal of $\mathcal{R}$ 
    to reposition the UAV to cover all required edges.  For an inspection walk $\gamma$ starting at $v_q$, let $D(v_q,\gamma)$ denote this extra deadheading length. (A walk is a sequence of adjacent vertices and edges where vertices or edges could be repeated). The corresponding time is
    \[
    T_{\mathrm{dead}}(v_q,\gamma) = \frac{D(v_q, \gamma)}{v_A}.
    \]
\end{enumerate}

\paragraph{Optimization problem}
The total UAV time is
\[
T_{\mathrm{UAV}}(j,x_A,v_q,\gamma) 
= T_{\mathrm{init}}(x_A,v_q) + T_{\mathrm{inspect}} + T_{\mathrm{dead}}(v_q,\gamma).
\]

To find $v_q$ and $\gamma$, the optimization problem is
\begin{equation}
\label{eq:subprob2-final}
\begin{split}
T_{\text{UAV}}^{\star}(j, x_A, v_q,\gamma) \;=\;
\min_{v_q} \;
\min_{\gamma \in \mathcal{W}(v_q)} 
\Big\{ T_{\mathrm{init}}(x_A,v_q) + \\
+ T_{\mathrm{inspect}} + T_{\mathrm{dead}}(v_q,\gamma) \Big\},
\end{split}
\end{equation}
where $\gamma \in \mathcal{W}(v_q)$ and $\mathcal{W}(v_q)$ is the set of feasible inspection walks on $E_\mathcal{R}$ starting at $v_q$ and covering all required edges $E_{\mathcal{R}}$. 
\begin{figure}[ht]
    \centering
    \begin{tikzpicture}[scale=1.0,>=stealth, node distance=1cm]
        \tikzstyle{vertex}=[circle,draw,fill=skyblue!40,minimum size=6mm,inner sep=0pt]
        
        \node[fill=none,below=10pt] (UAV_label) at (-1,2) {\textbf{UAV}};
        \node[vertex] (UAV) at (-1,2) {$x_A$};
        \node[vertex] (u0)  at (3.5,3) {$v_0$};
        \node[vertex] (u1)  at (2.5,3) {$v_1$};
        \node[vertex] (u2)  at (1,2) {$v_2$};
        \node[vertex] (u3) at (2,1) {$v_3$};
        \node[vertex] (u4) at (3,1) {$v_4$};
        \node[vertex] (u5)  at (4,1.5) {$v_5$};

        \draw[thick,blue,dashed] (UAV) -- (u2);
        \draw[thick,blue] (u0) -- (u1) -- (u2) -- (u3) -- (u4) -- (u5);

        \draw[thick,blue,dotted] (u0) -- (u5);

    \end{tikzpicture}
    \caption{Illustration of the optimal plan for the UAV to inspect path \nn{$\mathcal{R} = \langle v_0, v_1, \dots, v_5 \rangle$}. Dashed lines represent the initial transit path, solid lines represent the required inspection path, and dotted lines represent the deadheading paths.}
    \label{fig:optimal-uav-plan}
\end{figure}

Figure~\ref{fig:optimal-uav-plan} gives an example of the optimal plan for the UAV to inspect the path \nn{$\mathcal{R} = \langle v_0, v_1, \dots, v_5 \rangle$} when the UAV goes to $v_2$ first, inspects two edges $(v_2, v_1), (v_1, v_0)$, deadheads along $(v_0, v_5)$, continues inspection of edges $(v_5, v_4), (v_4, v_3),$ and $(v_3, v_2)$. In this case, the UAV will stop at $v_2$.

\section{UGV-only: Worst-Case Analysis}
\label{sec:analysis-ugv-only}

We now introduce the simplified theoretical model used for the worst-case analysis. The assumptions in this section are introduced for proving competitive ratio bounds for UGV worst-case travel time. These assumptions are not used in the Optimal Partition algorithm or in the experiments on real road networks.

We first consider the case of a UGV alone without UAV support, focusing on graphs where $k$ disjoint paths connect $s$ and $g$. The UGV must travel from $s$ to $g$, as in the original Canadian Traveller Problem. We derive a competitive ratio of $2k-1$.

\begin{figure}[ht]
    \centering
    \begin{tikzpicture}[
        >=Latex,
        every node/.style={circle, draw, inner sep=1.2pt, minimum size=7pt, transform shape},
        path/.style={line width=0.9pt},
        labelnode/.style={rectangle, draw, rounded corners, fill=white, inner sep=2pt, font=\scriptsize, transform shape}
    ]

        \node[fill=black, label=left:{$s$}, minimum size=8pt] (s) at (0,0) {};
        \node[fill=black, label=right:{$g$}, minimum size=8pt] (g) at (6, 0) {};

        \node (p1_1) at (2, -0.5) {};
        \node (p1_2) at (4, -0.5) {};
        \draw[path] (s) -- (p1_1) -- (p1_2) -- (g);
        \node[draw=none, fill=none, font=\scriptsize, above=0pt] at (3, -0.5) {$P_1$};

        \node (p2_1) at (2, 0.5) {};
        \node (p2_2) at (4, 0.5) {};
        \node (p2_3) at (4.5, 0.5) {};
        \draw[path] (s) -- (p2_1) -- (p2_2) -- (p2_3) -- (g);
        \node[draw=none, fill=none, font=\scriptsize, above=0pt] at (3, 0.5) {$P_2$};

        \node (p3_1) at (2, 1.25) {};
        \node (p3_2) at (4, 1.25) {};
        \draw[path] (s) -- (p3_1) -- (p3_2) -- (g);
        \node[draw=none, fill=none, font=\scriptsize, above=0pt] at (3, 1.25) {$P_3$};

        \node (p4_1) at (2, -1.25) {};
        \node (p4_2) at (5, -1.25) {};
        \draw[path] (s) -- (p4_1) -- (p4_2) -- (g);
        \node[draw=none, fill=none, font=\scriptsize, above=0pt] at (3, -1.25) {$P_4$};

        \node[labelnode, above left=3pt and -2pt of s] {UGV};
        \node[fill=black, minimum size=8pt] (s_A) at (4,0.5) {};
        \node[labelnode, below=6pt] at (4, 0.5) {$\text{UAV at } x_{A}$};

    \end{tikzpicture}
    \caption{Illustration of a graph with disjoint paths.}
    \label{fig:disjoint-paths}
\end{figure}

\subsection{Disjoint Paths Setting}
We first focus on the case where $s$ and $g$ are connected by $k$ disjoint paths $ \mathcal{P} = \{P_1, P_2, \dots, P_k\}, \quad P_i \cap P_j = \{s,g\} (\forall i \neq j)$, with paths sorted by their lengths so $L_1 \leq L_2 \leq \cdots \leq L_k$. See Figure~\ref{fig:disjoint-paths}.

\subsection{Adversary Model}
We extend the adversary model from the original Canadian Traveller Problem~\cite{PapadimitriouY91_BasicCTP}. The adversary selects which edges are blocked and where, prior to the start of robot motion, with the objective of maximizing the UGV’s travel cost before reaching $g$. To cause the worst case for $k$ disjoint paths, the adversary blocks the final edge of each of the first $(k-1)$ paths near $g$, forcing the UGV to repeatedly traverse and backtrack on these paths until reaching the $k$th (i.e., longest) path.

\subsection{Competitive Ratio}
Let $\pi$ denote a UGV-only or cooperative UGV--UAV strategy. The total travel time of the UGV under $\pi$ and a given instance $I$ of blocked edges is denoted $T_{UGV}(\pi, I)$. Our goal is to design a strategy that minimizes  $T_{UGV}$. The worst-case performance of $\pi$ over all instances can also be characterized by the \textit{competitive ratio}
\[
\rho(\pi) = \sup_{\{I\}} \frac{T_{UGV}(\pi, I)}{L^*(I)/v_G}
\]
where \(L^*(I)\) is the length of the shortest feasible \(s\)-\(g\) path in instance \(I\).


\subsection{Shortest-First Strategy}
Let $\mathcal{P}=\{P_1,\dots,P_k\}$ be the disjoint $s$--$g$ paths with lengths $L_1\le L_2\le\cdots\le L_k$ and let $S=\sum_{i=1}^k L_i$.

\begin{lemma}
Among all permutations of $\mathcal{P}$, attempting paths in nondecreasing length order $P_1,P_2,\dots,P_k$ minimizes the UGV’s worst-case travel cost.
\end{lemma}
\begin{proof}
For any order $\sigma$, an adversary can block the first $k-1$ attempted paths near $g$, causing the UGV to backtrack, and leave the last attempted path open, forcing cost
\[
C_\sigma \;=\; 2\sum_{i=1}^{k-1} L_{\sigma(i)} + L_{\sigma(k)} \;=\; 2S - L_{\sigma(k)}.
\]
The factor of $2$ appears because once a blockage is detected, the UGV must return to the start $s$ before attempting another path. Since $S$ is fixed, this quantity is minimized when $L_{\sigma(k)}$ is as large as possible; for any $i<k$,
\[
(2S-L_i)-(2S-L_k)=L_k-L_i\ge0,
\]
so placing $P_k$ last gives the smallest cost. Hence, the shortest-first strategy minimizes worst-case travel.
\end{proof}

We therefore use the shortest-first order $P_1,P_2,\dots,P_k$ for the UGV. An adversary that knows this policy maximizes wasted exploration by blocking the last edge of $P_1,\dots,P_{k-1}$ (as close to $g$ as possible) and leaving $P_k$ open, yielding the worst-case cost
\[
C^{*}\;=\;2\sum_{i=1}^{k-1} L_i \;+\; L_k \;=\; 2S - L_k.
\]

\subsection{Competitive Ratio of Shortest-First Strategy}
In the shortest-first strategy, the UGV attempts $P_1,P_2,\dots,P_k$ in nondecreasing length order. The adversary chooses an instance by deciding which paths to block and, for each blocked path, the location of the blockage. Since the paths are disjoint, attempting a path $P_i$ reveals no information about any other path. Moreover, for any blocked path $P_i$, placing the blockage as close to $g$ as possible maximizes the UGV's wasted travel on that path. Hence, in the worst case, we may assume each blocked path is blocked at its last edge (at a small distance $\epsilon_i$ from $g$).

Let $j\in\{1,\dots,k\}$ be the index of the first open path in the shortest-first order, i.e., $P_1,\dots,P_{j-1}$ are blocked and $P_j$ is open. Then the UGV traverses each blocked path nearly to $g$ and returns to $s$, and finally traverses $P_j$ once to reach $g$. Therefore, the UGV travel time is 
\[
T_{\mathrm{UGV}}(j)
=
\frac{2\sum_{i=1}^{j-1}(L_i-\epsilon_i)+L_j}{v_G}
\;\approx\;
\frac{2\sum_{i=1}^{j-1}L_i + L_j}{v_G}.
\]
The offline optimum, with full knowledge of the environment, takes the shortest open path, which in this instance is $P_j$, with time $L_j/v_G$. Hence, the competitive ratio of shortest-first is
\begin{flalign*}
\rho
&= \sup_{\text{instances}} \frac{T_{\mathrm{UGV}}}{\mathrm{OPT}}
= \max_{1\le j\le k}\frac{T_{\mathrm{UGV}}(j)}{L_j/v_G} &&\\
&\le \max_{1\le j\le k}\frac{2\sum_{i=1}^{j-1}L_i + L_j}{L_j}
= \max_{1\le j\le k}\left(1+2\frac{\sum_{i=1}^{j-1}L_i}{L_j}\right) &&
\end{flalign*}

Because $L_i\le L_j$ for all $i<j$, we have $\sum_{i=1}^{j-1}L_i \le (j-1)L_j$, and thus
\[
\rho
\;\le\;
\max_{1\le j\le k}\bigl(1+2(j-1)\bigr)
=
\max_{1\le j\le k}(2j-1)
=
2k-1.
\]

The bound is tight. Consider a graph with paths of equal-length $L_1=\cdots=L_k=L$ and an adversary that blocks
$P_1,\dots,P_{k-1}$ at their last edges and leaves only $P_k$ open. Then $j=k$ and
\[
\rho
=
\frac{2\sum_{i=1}^{k-1}L_i + L_k}{L_k}
=
\frac{2(k-1)L + L}{L}
=
2k-1.
\]
Therefore, the competitive ratio of the shortest-first strategy is exactly $2k-1$.

\section{Cooperative UGV--UAV: Worst-Case Analysis}
\label{sec:analysis-ugv-uav}

We now analyze the cooperative case in which a UAV scout assists the UGV. The intuition is that while the UGV traverses a prefix of a path, the UAV can simultaneously inspect its suffix to reduce the UGV’s wasted backtracking.

We use the same disjoint-paths and oblivious adversary model introduced in Section~\ref{sec:analysis-ugv-only}. However, here the adversary, knowing the joint strategy of the UGV and UAV, can always place blockages at locations that force the robots to completely inspect each path before it is declared as blocked. Thus, in the worst case, each attempted path $P_i$ must still be fully inspected by the team. We additionally assume, in this section only, that the UAV's initial transit and deadheading costs are negligible in order to obtain a clean analytical bound. These costs are explicitly included later in the Optimal Partition algorithm.


\subsection{Setup}
Let $P_1, \dots, P_k$ be the $k$ disjoint $s$--$g$ paths ordered by length $L_1 \le L_2 \le \cdots \le L_k$. Knowing the UGV--UAV strategy, the adversary can place the blockage so that neither robot can terminate early; i.e., both the prefix (UGV) and suffix (UAV) must be inspected in full.

\subsection{Cost of a Blocked Path}
Consider a blocked path $P_i$ with length $L_i$. The UGV and UAV may divide the inspection in various ways (e.g., prefix/suffix, UAV moving forward from $s$, UAV moving backward from $g$, or UAV starting in the middle). Yet, since the adversary knows the strategy, they can always place the blockage so that the entire path $P_i$ must be inspected before it is declared blocked.
The benefit of cooperation is that the workload is shared: the UGV need not inspect the entire path, as the UAV covers part of it. If the UGV alone inspects $P_i$, the time is $L_i/v_G$. If the UGV and UAV inspect disjoint segments in parallel, the time is $L_i/(v_G+v_A)$, so reducing the UGV’s portion of the cost in proportion to the speed ratio $\tfrac{v_G}{v_G+v_A}$.
Formally, the UGV’s travel distance for blocked path $P_i$ is
$\frac{2v_G}{v_G+v_A} \, L_i,$
compared to $2L_i$ in the UGV-only case. 


\subsection{Total Time Under Shortest-First Cooperation}
If the first open path is $P_j$, then the UGV incurs the above blocked-path cost for $P_1,\dots,P_{j-1}$ and then traverses $P_j$ to reach $g$. Hence, the UGV travel time is
\[
T_{\mathrm{UGV\text{--}UAV}}(j)
\;=\;
\sum_{i=1}^{j-1}\frac{2L_i}{v_G+v_A} \;+\; \frac{L_j}{v_G}
\]

The offline optimum, knowing the instance, directly uses the shortest open path, which in this construction is $P_j$, with time $\mathrm{OPT}(j)=L_j/v_G$.

\subsection{Competitive Ratio}
Using the shortest-first cooperative strategy, the competitive ratio is the maximum over all adversarial instances, or equivalently, over all possible indices $j$ of the first open path:
\begin{flalign*}
\rho_{\mathrm{UGV\text{--}UAV}}
&=
\sup_{\text{instances}}\frac{T_{\mathrm{UGV\text{--}UAV}}}{\mathrm{OPT}}
=
\max_{1\le j\le k}
\frac{T_{\mathrm{UGV\text{--}UAV}}(j)}{L_j/v_G} &&\\
&=
\max_{1\le j\le k}
\left(
\frac{v_G}{L_j}\sum_{i=1}^{j-1}\frac{2L_i}{v_G+v_A}
+
1
\right) &&\\
&\;=\;
\max_{1\le j\le k}
\left(
1+\frac{2v_G}{v_G+v_A}\frac{\sum_{i=1}^{j-1}L_i}{L_j}
\right).
\end{flalign*}
Since $L_i\le L_j$ for all $i<j$, we have $\sum_{i=1}^{j-1}L_i\le (j-1)L_j$, and therefore
\begin{flalign*}
\rho_{\mathrm{UGV\text{--}UAV}}
\;\le\;
\max_{1\le j\le k}\left(1+\frac{2v_G}{v_G+v_A}(j-1)\right) &&\\
=
1+\frac{2v_G}{v_G+v_A}(k-1)
=
2(k-1)\frac{v_G}{v_G+v_A} + 1.
\end{flalign*}
The bound is tight: if $L_1=\cdots=L_k$,  the adversary blocks $P_1,\dots,P_{k-1}$ (forcing full path inspection) and leaves only $P_k$ open, then equality holds.




\begin{theorem}
\label{thm:coop-bound}
In the disjoint-path Canadian Traveller Problem under the simplifying assumption that the UAV’s initial transit and deadheading costs are negligible, if a UAV assists the UGV and adversarially placed blockages force complete inspection of each attempted path, the worst-case competitive ratio is
\[
\rho_{\text{UGV--UAV}} = 2(k-1)\frac{v_G}{v_G+v_A} + 1.
\]
\end{theorem}

\textbf{Remark}: The simplifying assumption is made only to obtain a clean analytical bound. In Section~\ref{sec:partition}, we show how the UAV actively chooses its starting point and inspection plan to minimize initial transit and deadheading costs, to minimize  its travel costs in practice. 

\section{Optimal UAV Inspection of a Path (Subproblem~2)}
\label{sec:partition}
We now consider computing the path for a UAV to inspect a given path suffix in a general graph $G$. Let \nn{$\mathcal{R}=\langle v_0,\dots,v_\ell \rangle$} denote the path suffix for the UAV to inspect. It has a vertex set  $V_{\mathcal{R}} = \{v_0,\dots,v_\ell\}$, required edge set $E_{\mathcal{R}} = \{(v_{r-1},v_r)\}_{r=1}^{\ell}$, and total length $L_{\mathcal{R}}=\sum_{r=1}^{\ell} w(v_{r-1},v_r)$.


The UAV begins at $x_A\in V$, flies on the complete motion graph $G_A$ with edge costs $w_A$, and its total time is
\[
\begin{split}
T_{\mathrm{UAV}}(x_A,v_q,\gamma)
= T_{\mathrm{init}}(x_A,v_q) + T_{\mathrm{inspect}} + T_{\mathrm{dead}}(v_q,\gamma) \\
= \frac{w_A(x_A,v_q)}{v_A}
  + \frac{L_{\mathcal{R}}}{v_A}
  + \frac{D(v_q,\gamma)}{v_A}
\end{split}
\]
where $v_q\in V_{\mathcal{R}}$ is the chosen start vertex on $\mathcal{R}$ and
$D(v_q,\gamma)$ is the \textit{deadheading along $\mathcal{R}$} incurred by the inspection walk $\gamma$ beyond the mandatory length $L_{\mathcal{R}}$ of required edges.

\subsection{$T$-join Formulation for Minimum Deadheading}


To inspect all edges of the required inspection subgraph $H=(V_\mathcal{R}, E_\mathcal{R})$ induced by $\mathcal{R}$, the UAV may need to perform deadheading on edges in the motion graph $G_A$.
This setting is closely related to the Chinese Postman Problem~\cite{Guan1962_CPP,EdmondsJ1973_CPPMatching}, except that our inspection walk is allowed to start and stop at different vertices and the additional edges are  chosen from $G_A$. We compute the minimum deadheading needed to realize such an inspection walk using a minimum-weight $T$-join~\cite{EdmondsJ1973_CPPMatching}.

Let $T \subseteq V$ be a subset of vertices of the graph $G=(V, E, w)$. A \textit{$T$-join} is a subset of edges $J \subseteq E_A$ such that, in the graph $(V, J)$, every vertex in $T$ has odd degree, while all other vertices have even degree. Equivalently, adding the set $J$ flips the degree parity of exactly the vertices in $T$. By carefully choosing $T$, we can augment the required inspection subgraph by forming $H' = (V_{\mathcal R},\, E_{\mathcal R} \cup J)$, where $J$ represents the deadheading edges. The graph $H'$ admits an Euler trail from a predetermined start vertex $v_i$ to a predetermined $v_j$ when its odd-degree vertices are exactly $\{v_i,v_j\}$ (in the special case $v_i=v_j$, all vertices have even degree). Therefore, choosing $T$ to enforce the appropriate parity condition and minimizing deadheading reduces to the minimum-weight $T$-join problem. 

Figure~\ref{fig:modified-T-join} provides an example of the T-join problem when the start vertex is $v_2$ and the stop vertex is $v_\ell$. In the original path $\mathcal{R}$, all interior vertices (i.e., $v_1, v_2$) have even degree, while the endpoints $v_0, v_\ell$ have odd degree. If we want the final Eulerian path to start at $v_2$ and end at $v_\ell$, the degree of all vertices except $v_2$ and $v_\ell$ has to be even. So we need to flip the degree of the two vertices $v_2$ and $v_0$, that is, set $T=\{v_0, v_2\}$. To achieve this optimally, we add a shortest path (i.e., a deadheading edge) between $v_0$ and $v_2$ to the original graph, forming a T-join that enables an Eulerian inspection path.

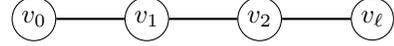
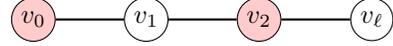
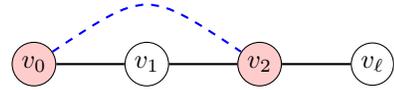
\begin{figure}[!ht]
    \centering
    \begin{subfigure}{0.5\textwidth}
        \centering
        \begin{tikzpicture}[scale=1.0, every node/.style={circle,draw,inner sep=2pt}]
            \node (v0) at (0,0) {$v_0$};
            \node (v1) at (1.5,0) {$v_1$};
            \node (v2) at (3,0) {$v_2$};
            \node (v3) at (4.5,0) {$v_\ell$};
            \draw[thick] (v0)--(v1)--(v2)--(v3);
        \end{tikzpicture}
        \caption{Required path $\mathcal{R}$ with endpoints of odd degree.}
        \label{fig:tjoin_a}
    \end{subfigure}
    \hfill

    \begin{subfigure}{0.45\textwidth}
        \centering
        \begin{tikzpicture}[scale=1.0, every node/.style={circle,draw,inner sep=2pt}]
            \node[fill=red!20] (v0) at (0,0) {$v_0$};
            \node (v1) at (1.5,0) {$v_1$};
            \node[fill=red!20] (v2) at (3,0) {$v_2$};
            \node (v3) at (4.5,0) {$v_\ell$};
            \draw[thick] (v0)--(v1)--(v2)--(v3);
        \end{tikzpicture}
        \caption{The shaded vertices form the set $T = \{v_0, v_2\}$ of vertices whose degree has to be flipped if the UAV starts at $v_2$ and stops at $v_\ell$.}
        \label{fig:tjoin_b}
    \end{subfigure}
    \hfill

    \begin{subfigure}{0.5\textwidth}
        \centering
        \begin{tikzpicture}[scale=1.0, every node/.style={circle,draw,inner sep=2pt}]
            \node[fill=red!20] (v0) at (0,0) {$v_0$};
            \node (v1) at (1.5,0) {$v_1$};
            \node[fill=red!20] (v2) at (3,0) {$v_2$};
            \node (v3) at (4.5,0) {$v_\ell$};
            \draw[thick] (v0)--(v1)--(v2)--(v3);
            \draw[dashed,blue,thick] (v0) .. controls (1.5,1) .. (v2);
        \end{tikzpicture}
        \caption{Adding the edge $J$ (dashed) makes the graph Eulerian.}
        \label{fig:tjoin_c}
    \end{subfigure}

    \begin{subfigure}{0.32\textwidth}
    \end{subfigure}
    \caption{Illustration of the T-join problem to create an Eulerian path when the start and the stop vertices are determined.}
    \label{fig:modified-T-join}
\end{figure}

\subsection{Selecting a Start Vertex for UAV Inspection}
\label{sec:start-vertex}

The UAV's current position influences the start vertex selected in $\mathcal{R}$ for the optimal inspection plan. This affects the initial transit and deadheading costs also. Consider $\mathcal{R}= \langle v_0, v_1, \dots, v_5 \rangle$ in Figure~\ref{fig:uav-3-start-cases}. If the UAV's position is $x_1$, the optimal path will start at $v_2$, and use one deadheading edge $(v_0, v_5)$. But if the UAV's position is $x_2$, the optimal path will start at $v_0$ and stop at $v_5$, without requiring any deadheading edges.

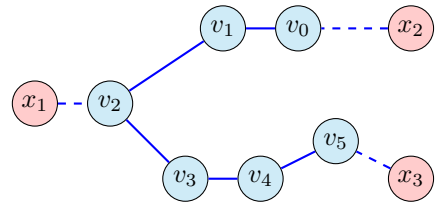
\begin{figure}[H]
    \centering
    \begin{tikzpicture}[scale=1.0,>=stealth, node distance=1cm]
        \tikzstyle{vertex}=[circle,draw,fill=skyblue!40,minimum size=6mm,inner sep=0pt]
        
        \node[vertex,fill=red!20] (x1) at (0,2) {$x_1$};
        \node[vertex,fill=red!20] (x2)  at (5,3) {$x_2$};
        \node[vertex,fill=red!20] (x3)  at (5,1) {$x_3$};
        
        \node[vertex] (u0)  at (3.5,3) {$v_0$};
        \node[vertex] (u1)  at (2.5,3) {$v_1$};
        \node[vertex] (u2)  at (1,2) {$v_2$};
        \node[vertex] (u3) at (2,1) {$v_3$};
        \node[vertex] (u4) at (3,1) {$v_4$};
        \node[vertex] (u5)  at (4,1.5) {$v_5$};

        \draw[thick,blue,dashed] (x1) -- (u2);
        \draw[thick,blue,dashed] (x2) -- (u0);
        \draw[thick,blue,dashed] (x3) -- (u5);
        \draw[thick,blue] (u0) -- (u1) -- (u2) -- (u3) -- (u4) -- (u5);


    \end{tikzpicture}
    \caption{Illustrating the three cases for the UAV's current position.}
    \label{fig:uav-3-start-cases}
\end{figure}

\subsubsection{Case 1: Start at $v_0$}
If the UAV starts at $v_0$, the optimal inspection walk is a single sweep from $v_0$ to $v_\ell$. No additional deadheading edges are required, hence $D(v_0,\gamma) = 0$, and
\[
T_{\mathrm{UAV}}(x_A,v_0,\gamma^\star) = \frac{w_A(x_A,v_0)}{v_A} + \frac{L_{\mathcal{R}}}{v_A}.
\]

\subsubsection{Case 2: Start at $v_\ell$}
Symmetrically, if the UAV starts at $v_\ell$, the optimal walk is a single sweep from $v_\ell$ to $v_0$, yielding $D(v_l,\gamma) = 0$ and
\[
T_{\mathrm{UAV}}(x_A,v_\ell,\gamma^\star) = \frac{w_A(x_A,v_\ell)}{v_A} + \frac{L_{\mathcal{R}}}{v_A}.
\]

\subsubsection{Case 3: Start at an interior vertex $v_i$ ($0<i<\ell$)}
When the UAV starts at an interior vertex $v_i$, this start vertex becomes one of the two odd-degree vertices required for an Eulerian trail. Let us examine the choice of stop vertex:

\begin{itemize}
    \item \textbf{Case 3.1 - Stop at $v_0$:} Here, the set of vertices whose degree must be flipped is $T = \{v_i, v_\ell\}$. The minimum-weight $T$-join corresponds to the shortest path connecting $v_i$ and $v_\ell$, with cost $w_A(v_i, v_\ell)$, which determines the optimal additional deadheading $D(v_0,\gamma)$.
    \item \textbf{Case 3.2 - Stop at $v_\ell$:} Symmetrically, $T = \{v_i, v_0\}$, and the minimum-weight $T$-join is the shortest path connecting $v_i$ and $v_0$ with cost $w_A(v_0, v_i)$.
    \item \textbf{Case 3.3 - Stop at $v_i$ (return to start):} In this case, $T = \{v_0, v_\ell\}$, and the minimum-weight $T$-join is the shortest path connecting the endpoints $v_0$ and $v_\ell$ or or $w_A(v_0, v_\ell)$.
    \item \textbf{Case 3.4 - Stop at any other vertex $v_j \notin \{v_0,v_i,v_\ell\}$:} Then $T = \{v_0,v_\ell,v_i,v_j\}$, and the minimum-weight $T$-join involves at least two edges. By the properties of $T$-joins, its total weight is strictly larger than any of the three previous cases, making such a stop vertex suboptimal.

    Table~\ref{tab:uav_start_i_stop_j} enumerates all possible ways to construct a $T$-join when the UAV starts at an interior vertex $v_i$ and stops at a vertex $v_j \notin \{v_0, v_i, v_\ell\}$, forming the set $T = \{v_0, v_i, v_j, v_\ell\}$. Each case requires adding at least two deadheading edges to pair off the four vertices. However, in every configuration, the total deadheading cost exceeds that of a simpler $T$-join involving only one additional edge, such as connecting $(v_0, v_i)$, $(v_i, v_\ell)$, or $(v_0, v_\ell)$ individually. Therefore, any strategy that stops at $v_j \notin \{v_0, v_i, v_\ell\}$ is suboptimal, and the optimal solution must involve stopping at one of the three special vertices: the same start vertex $v_i$, or one of the endpoints $v_0$ or $v_\ell$.
\end{itemize}

\begin{table}[!ht]
\centering
\caption{All possible ways to add deadheading for the set $T=\{v_0, v_i, v_j, v_\ell\}$.}
\begin{tabular}{lc}
\hline
Deadheading edges & Cost  \\
\hline
$(v_0, v_i), (v_j, v_\ell)$     &       $w_A(v_0, v_i) + w_A(v_j, v_\ell) > w_A(v_0, v_i)$\\
$(v_0, v_j), (v_i, v_\ell)$     &       $w_A(v_0, v_j) + w_A(v_i, v_\ell) > w_A(v_i, v_\ell)$\\
$(v_0, v_\ell), (v_i, v_j)$     &       $w_A(v_0, v_\ell) + w_A(v_i, v_j) > w_A(v_0, v_\ell)$\\
\hline
\end{tabular}
\label{tab:uav_start_i_stop_j}
\end{table}

\section{Optimal Partition Algorithm}
\label{sec:algorithms}

We now present an optimal algorithm to efficiently partition the shortest path in general graphs (Algorithm~\ref{alg:optimal-split}), and analyze its computational complexity.

When the UGV or UAV detects a blockage on the current shortest path, the UGV computes a new shortest path from its current position $x_G$.
We must then compute the optimal split index $j^\star$ that minimizes the maximum of two costs: the UGV inspection time for the prefix and the UAV inspection time for the suffix. That is,
\[
j^\star = \arg\min_{j \in \{0, \dots, m\}} \max\{T_{\mathrm{UGV}}(j),\; T_{\mathrm{UAV}}^{\star}(j)\}.
\]

\begin{algorithm}[H]
\caption{Optimal Partition for UGV-UAV Path Inspection}
\label{alg:optimal-split}
\begin{algorithmic}[1] 
\Require Path $P = \langle u_0, u_1, \dots, u_m \rangle$ with edge weights $w(u_{r-1}, u_r)$
\Require UAV current position $x_A$, speeds $v_G$, $v_A$, and UGV return time $\tau_G^i$
\Ensure Optimal split index $j^\star$
\For{$j = 0$ to $m$}
    \State $S(j) \gets \sum_{r=1}^{j} w(u_{r-1}, u_r)$ \Comment{Prefix length}
    \State $T_{\mathrm{UGV}}(j) \gets \tau_G^i + S(j)/v_G$
    \State Define suffix \nn{$\mathcal{R} = \langle u_j, u_{j+1}, \dots, u_m\rangle$}
    \State Compute $T_{\mathrm{UAV}}^{\star}(j)$ using Eq.~\eqref{eq:subprob2-final}
    \State $T(j) \gets \max(T_{\mathrm{UGV}}(j),\; T_{\mathrm{UAV}}^{\star}(j))$
\EndFor
\State $j^\star \gets \arg\min_j T(j)$
\State \Return $j^\star$
\end{algorithmic}
\end{algorithm}

\subsection{Complexity Analysis}
Let $m$ be the number of edges in the path \nn{$P = \langle u_0, u_1, \dots, u_m\rangle$}, where $P$ is the shortest path from the UGV's current location to the goal vertex $g$.
\begin{itemize}
    \item \textbf{The prefix cost $T_{\mathrm{UGV}}(j)$:} it can be computed in $\mathcal{O}(1)$ time per $j$ using prefix sums, after an initial $\mathcal{O}(m)$ precomputation.
    \item \textbf{The suffix cost $T_{\mathrm{UAV}}^{\star}(j)$:} For each split index \(j\), we must compute the optimal UAV inspection time over the suffix subpath \nn{$\mathcal{R} = \langle u_j, u_{j+1}, \dots, u_m\rangle$}. This reduces to the minimum-weight $T$-join problem, where we search for the best choice of start and stop vertices for the UAV. Specifically, we consider three cases (Section~\ref{sec:start-vertex}):
    \begin{enumerate}
        \item Stopping at $u_j$ (one endpoint),
        \item Stopping at $u_m$ (the other endpoint),
        \item Returning to the same interior start vertex $u_i$ with $j < i < m$.
    \end{enumerate}
    For each case, the deadheading distance is the Euclidean distance, computed in constant time per pair. To compute the best inspection path, we iterate over all candidate starting vertices in the suffix $\mathcal{R}$, of which there are at most $m - j + 1$. Therefore, the total time to compute $T_{\mathrm{UAV}}^{\star}(j)$ is $\mathcal{O}(m - j + 1) = \mathcal{O}(m)$.
\end{itemize}

\textbf{Total cost:} The algorithm evaluates $m+1$ candidate splits. For each split index $j$, computing the prefix cost takes constant time, and computing the suffix cost takes $\mathcal{O}(m)$. Thus, the total runtime is
\[
\sum_{j=0}^{m} \mathcal{O}(m) = \mathcal{O}(m^2).
\] 

\section{Results}
\label{sec:results}
\subsection{Datasets}
We conducted experiments on the Line Coverage \cite{AgarwalA24} dataset, which is derived from OpenStreetMap road networks of the 50 most populous cities worldwide. Each map is represented at two spatial scales: a small area of size 1 km $\times$ 1 km, and a large area of size 3 km $\times$ 3 km.
The 50 small road network graphs have, on average, 334.7 vertices and 371.0 edges, while the 50 large road network graphs have, on average, 1813.4 vertices and 1897.4 edges.

\begin{figure}[ht]
    \centering
    \includegraphics[width=0.8\linewidth]{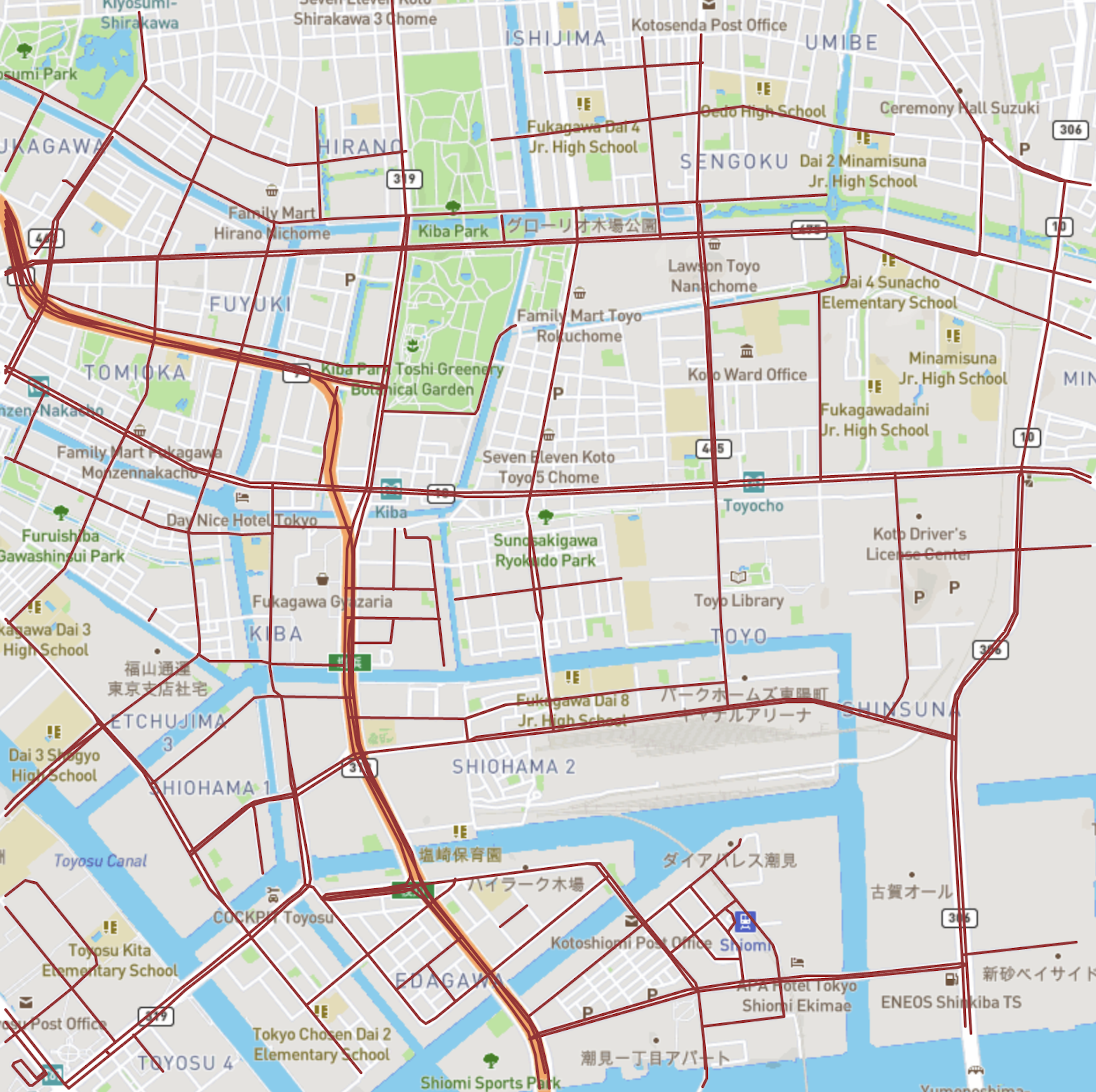}
    \caption{Representative large road network graph from the Tokyo map in the Line Coverage dataset. The graph contains 2,171 vertices and 2,323 edges.}
    \label{fig:large_city_map}
\end{figure}

Figure~\ref{fig:large_city_map} shows a representative large road network graph from Tokyo. For each road network, we generate \(50\) unique instances by randomly selecting start and goal nodes for the UGV, as well as a random start node for the UAV. We ensure that at least one viable path exists between the UGV's start and goal nodes. 
We fix the UGV speed $v_G$ to 20 m/s and evaluate the UAV at three speeds \(v_A\in\{20,30,40\}\) m/s.


For the experiments, we assume that:
\begin{itemize}
    \item For each blocked edge, the damaged point is sampled randomly along the edge to model uncertain road damage locations.
    \item The UGV and UAV can only send or receive information about damaged points when they are at a vertex. 
    \item Path reversals for the UGV and UAV are permitted only at vertices and damage points. Hence, the UAV cannot halt at mid-edge. If it is already traversing an edge when the UGV encounters a damaged point, it must continue to the next vertex, which may require the UGV to wait after it comes back to the most recently visited vertex.
\end{itemize}

\subsection{Strategies}
We compared the following strategies:
\begin{enumerate}
    \item \textbf{Full Observation:} The UGV has perfect knowledge of damaged and safe edges. This case serves as a lower bound for the problem.
    \item \textbf{UGV Only:} The UGV repeatedly computes shortest paths based on its partial knowledge. When it encounters a damaged edge, it returns to the last visited vertex and replans a new path.
    \item \textbf{UAV Bidirectional:} After the UGV selects a new shortest path, the UAV inspects that path backwards from the last uninspected endpoint on the path, thereby assisting in inspecting the UGV's current path.
    \item \textbf{Optimal Partition:} Our proposed strategy partitions the inspection task between the UAV and UGV optimally after the UGV's path is selected.
\end{enumerate}

\subsection{Results}
To provide a representative overview, we report detailed results for 10 major cities with a $v_G$:$v_A$ speed ratio of 20:40. These cities are the two most populous cities from five continents: Tokyo and Delhi (Asia), Lagos and Cairo (Africa), Moscow and Istanbul (Europe), Mexico City and New York (North America), and São Paulo and Buenos Aires (South America). These cities span diverse geographies and road network structures, making them suitable benchmarks for evaluating the performance of strategies.


Tables~\ref{tab:small_results_top} and~\ref{tab:large_results_top} summarize the average UGV traversal times for the selected cities. In nearly all cases, the Optimal Partition strategy outperforms both the UGV-only and UAV Bidirectional strategies. 
The full results for all 50 cities and all speed settings are reported in the supplementary material. 

\begin{table}[h]
\centering
\caption{Average UGV traversal time (in seconds) for ten representative small maps with $v_G$:$v_A$ at 20:40 m/s.}
\begin{tabular}{lcccc}
\hline
 & Full Obs. & UGV-only & Bidir. & Optimal Partition \\
\hline
Tokyo & 41.851 & 57.929 & 53.458 & \textbf{51.458} \\ 
Delhi & 57.040 & 80.874 & 69.492 & \textbf{66.239} \\ 
Lagos & 46.471 & 78.890 & 70.653 & \textbf{55.944} \\ 
Cairo & 43.498 & 67.404 & 54.422 & \textbf{49.539} \\ 
Moscow & 39.719 & 52.212 & 49.359 & \textbf{45.743} \\ 
Istanbul & 49.104 & 82.663 & 78.560 & \textbf{63.820} \\ 
Mexico City & 53.809 & 85.499 & 84.243 & \textbf{69.002} \\ 
New York & 39.116 & 54.794 & 50.185 & \textbf{45.515} \\ 
Sao Paulo & 42.558 & 65.610 & 58.315 & \textbf{51.542} \\ 
Buenos Aires & 49.236 & 69.212 & 63.288 & \textbf{58.399} \\ 

\hline
\end{tabular}
\label{tab:small_results_top}
\end{table}

\begin{table}[h]
\centering
\caption{Average UGV traversal time (in seconds) for ten representative large maps with $v_G$:$v_A$ at 20:40 m/s.}
\begin{tabular}{lcccc}
\hline
 & Full Obs. & UGV-only  & Bidir. & Optimal Partition \\
\hline
Tokyo & 127.992 & 201.025 & 187.344 & \textbf{149.688} \\ 
Delhi & 133.553 & 184.324 & 161.905 & \textbf{151.423} \\ 
Lagos & 155.460 & 206.360 & 180.240 & \textbf{178.965} \\ 
Cairo & 140.656 & 200.970 & 188.795 & \textbf{159.126} \\ 
Moscow & 182.203 & 202.529 & 200.611 & \textbf{200.220} \\ 
Istanbul & 220.729 & 277.356 & 272.065 & \textbf{238.611} \\ 
Mexico City & 153.582 & 220.373 & 220.172 & \textbf{198.041} \\ 
New York & 147.011 & 192.929 & 176.034 & \textbf{162.174} \\ 
Sao Paulo & 136.393 & 186.323 & 167.940 & \textbf{158.571} \\ 
Buenos Aires & 152.827 & 184.308 & 171.515 & \textbf{165.561} \\
\hline
\end{tabular}
\label{tab:large_results_top}
\end{table}


Table~\ref{tab:speed_ratio_improvement} reports the average UGV travel-time reduction of Optimal Partition relative to the UGV-only baseline in all three speed settings in both small and large maps. The results show that Optimal Partition consistently improves UGV travel time across all tested speed settings. The improvement in these general graphs also increases as the UAV becomes faster.
\begin{table}[ht]
\caption{Average UGV travel-time reduction of Optimal Partition relative to UGV-only at different UGV:UAV speed ratios. Here $v_G$ is 20 m/s and \(v_A\in\{20,30,40\}\) m/s.}
\centering
\begin{tabular}{l|ccc}
\hline
Map size & \(20{:}20\) & \(20{:}30\) & \(20{:}40\) \\
\hline
Small maps & 7.05\% & 10.94\% & 12.87\% \\
Large maps & 14.30\% & 16.39\% & 17.97\% \\
\hline
\end{tabular}
\label{tab:speed_ratio_improvement}
\end{table}

\subsection{Limitations and Discussion}
Our implementation assumes that the UGV and UAV exchange blockage information only when the transmitting or receiving robot is at a graph vertex, and that the UAV completes traversal of its current edge before reporting inspection results. These assumptions simplify event handling because the graph does not need to be modified by inserting new vertices at discovered blockage points. However, they can introduce additional waiting time and may reduce the benefit of UAV assistance in some instances. This explains rare cases in which UGV-only slightly outperforms Optimal Partition.

\section{Conclusion and Future Work}
\label{sec:conclusion}

We studied the problem of shortest path planning for a UGV with cooperative inspection by a UAV in an uncertain environment, where some edges in the road network may be blocked. Our formulation extends the original Canadian Traveller Problem by incorporating UAV assistance to reduce the UGV's travel time to reach the goal.
Our analysis for graphs with $k$ disjoint paths shows that the competitive ratio reduces from $2k-1$ for a UGV  to $2(k-1)\frac{v_G}{v_G+v_A} +1$ for a UGV with UAV assistance.
We introduced an optimal partitioning strategy that assigns edge inspection tasks between the UGV and UAV along a candidate path, and presented an efficient algorithm to compute both the optimal split point and the corresponding UAV inspection plan. The algorithm has a time complexity quadratic in the number of path edges, making it practical for online real-time planning.

\medskip
There are many directions for future work. An immediate generalization is handling teams with multiple UAVs and multiple UGVs that may start and finish at different locations. This raises practical questions about how the robots coordinate: who inspects what, and how they share information. We plan to explore distributed or decentralized control strategies.
Incorporating physical constraints such as UAV battery limits, communication range, sensor field of view, collision avoidance, and vehicle dynamics  is an important direction for enabling field deployment.
%

\nn{
\section*{Acknowledgments}
We gratefully acknowledge support from UNC Charlotte.
}

\bibliographystyle{IEEEtran}
\bibliography{refs}

@Article{MunasinghePD24_UAVUGVReview,
    AUTHOR = {Munasinghe, Isuru and Perera, Asanka and Deo, Ravinesh C.},
    TITLE = {A Comprehensive Review of {UAV-UGV} Collaboration: Advancements and Challenges},
    JOURNAL = {Journal of Sensor and Actuator Networks},
    VOLUME = {13},
    YEAR = {2024},
    NUMBER = {6},
    ARTICLE-NUMBER = {81},
    ISSN = {2224-2708},
    DOI = {10.3390/jsan13060081}
}

@INPROCEEDINGS{KashyapGMSD19,
  author={Kashyap, Abhishek and Ghose, Debasish and Menon, Prathyush P. and Sujit, P.B. and Das, Kaushik},
  booktitle={2019 International Conference on Unmanned Aircraft Systems (ICUAS)}, 
  title={{UAV} Aided Dynamic Routing of Resources in a Flood Scenario}, 
  year={2019},
  pages={328-335},
  doi={10.1109/ICUAS.2019.8798357},
  ISSN={2575-7296},
  month=jun,
}

@inproceedings{ChoudhurySKP22_arxiv,
    author = {Choudhury, Shushman and Solovey, Kiril and Kochenderfer, Mykel and Pavone, Marco},
    title = {Coordinated Multi-Agent Pathfinding for Drones and Trucks over Road Networks},
    year = {2022},
    isbn = {9781450392136},
    publisher = {International Foundation for Autonomous Agents and Multiagent Systems},
    address = {Richland, SC},
    booktitle = {Proceedings of the 21st International Conference on Autonomous Agents and Multiagent Systems},
    pages = {272–280},
    numpages = {9},
    location = {Virtual Event, New Zealand},
    series = {AAMAS '22}
}

@article{PapadimitriouY91_BasicCTP,
title = {Shortest paths without a map},
journal = {Theoretical Computer Science},
volume = {84},
number = {1},
pages = {127-150},
year = {1991},
issn = {0304-3975},
doi = {https://doi.org/10.1016/0304-3975(91)90263-2},
author = {Christos H. Papadimitriou and Mihalis Yannakakis},
}

@inproceedings{BarNoyS91_kCTP,
    author = {Bar-Noy, Amotz and Schieber, Baruch},
    title = {The {Canadian Traveller Problem}},
    year = {1991},
    isbn = {0897913760},
    publisher = {Society for Industrial and Applied Mathematics},
    address = {USA},
    booktitle = {Proceedings of the Second Annual ACM-SIAM Symposium on Discrete Algorithms},
    pages = {261–270},
    numpages = {10},
    location = {San Francisco, California, USA},
    series = {SODA '91}
}

@inproceedings{NikolovaK08_StochasticCTP,
    author = {Nikolova, Evdokia and Karger, David R.},
    title = {Route planning under uncertainty: the {Canadian Traveller Problem}},
    year = {2008},
    isbn = {9781577353683},
    publisher = {AAAI Press},
    booktitle = {Proceedings of the 23rd National Conference on Artificial Intelligence - Volume 2},
    pages = {969–974},
    numpages = {6},
    location = {Chicago, Illinois},
    series = {AAAI'08}
}

@article{BnayaFFMS21_RepeatedCTP,
    title={Repeated-Task {Canadian Traveler Problem}},
    volume={2},
    DOI={10.1609/socs.v2i1.18197},
    number={1},
    journal={Proceedings of the International Symposium on Combinatorial Search}, 
    author={Bnaya, Zahy and Felner, Ariel and Fried, Dror and Maksin, Olga and Shimony, Solomon},
    year={2021},
    month={Aug.},
    pages={24-30}
}

@inproceedings{BnayaFS09_RemoteSensing,
    author = {Bnaya, Zahy and Felner, Ariel and Shimony, Solomon Eyal},
    title = {{Canadian Traveler Problem} with remote sensing},
    year = {2009},
    publisher = {Morgan Kaufmann Publishers Inc.},
    address = {San Francisco, CA, USA},
    booktitle = {Proceedings of the 21st International Joint Conference on Artificial Intelligence},
    pages = {437–442},
    numpages = {6},
    location = {Pasadena, California, USA},
    series = {IJCAI'09}
}

@ARTICLE{XuHSZZ09_CompetitiveAnalysis,
    title = {The {Canadian Traveller Problem} and its competitive analysis},
    author = {Xu, Yinfeng and Hu, Maolin and Su, Bing and Zhu, Binhai and Zhu, Zhijun},
    year = {2009},
    journal = {Journal of Combinatorial Optimization},
    volume = {18},
    number = {2},
    pages = {195-205},
}

@article{Westphal08_kCTPAnalysis,
title = {A note on the k-{Canadian Traveller Problem}},
journal = {Information Processing Letters},
volume = {106},
number = {3},
pages = {87-89},
year = {2008},
issn = {0020-0190},
doi = {https://doi.org/10.1016/j.ipl.2007.10.004},
author = {Stephan Westphal},
}

@InProceedings{BampisEX22_CTPPredictions,
    author="Bampis, Evripidis
    and Escoffier, Bruno
    and Xefteris, Michalis",
    editor="Chalermsook, Parinya
    and Laekhanukit, Bundit",
    title="{Canadian Traveller Problem} with Predictions",
    booktitle="Approximation and Online Algorithms",
    year="2022",
    publisher="Springer International Publishing",
    address="Cham",
    pages="116--133",
    isbn="978-3-031-18367-6"
}

@article{PolychronopoulosT96_SSPPR,
    author = {Polychronopoulos, George H. and Tsitsiklis, John N.},
    title = {Stochastic shortest path problems with recourse},
    journal = {Networks},
    volume = {27},
    number = {2},
    pages = {133-143},
    doi = {https://doi.org/10.1002/(SICI)1097-0037(199603)27:2<133::AID-NET5>3.0.CO;2-L},
    year = {1996}
}

@incollection{DemetrescuEGI10,
    author = {Demetrescu, Camil and Eppstein, David and Galil, Zvi and Italiano, Giuseppe F.},
    title = {Dynamic graph algorithms},
    booktitle = {Algorithms and Theory of Computation Handbook: General Concepts and Techniques},
    edition = {2},
    publisher = {Chapman \& Hall/CRC},
    year = {2010},
    isbn = {9781584888222},
}

@article{HanauerHS22,
    author = {Hanauer, Kathrin and Henzinger, Monika and Schulz, Christian},
    title = {Recent Advances in Fully Dynamic Graph Algorithms – A Quick Reference Guide},
    year = {2022},
    issue_date = {December 2022},
    publisher = {Association for Computing Machinery},
    address = {New York, NY, USA},
    volume = {27},
    issn = {1084-6654},
    doi = {10.1145/3555806},
    journal = {ACM J. Exp. Algorithmics},
    month = dec,
    articleno = {1.11},
    numpages = {45}
}

@article{FriedSBW13_CTPComplexity,
title = {Complexity of {Canadian Traveler Problem} variants},
journal = {Theoretical Computer Science},
volume = {487},
pages = {1-16},
year = {2013},
issn = {0304-3975},
doi = {https://doi.org/10.1016/j.tcs.2013.03.016},
author = {Dror Fried and Solomon Eyal Shimony and Amit Benbassat and Cenny Wenner}
}

@article{BenderW12_OptimalRandomkCTP,
    author = {Bender, Marco and Westphal, Stephan},
    title = {An optimal randomized online algorithm for the k-{Canadian Traveller Problem} on node-disjoint paths},
    year = {2015},
    issue_date = {July      2015},
    publisher = {Springer-Verlag},
    address = {Berlin, Heidelberg},
    volume = {30},
    number = {1},
    issn = {1382-6905},
    doi = {10.1007/s10878-013-9634-8},
    journal = {J. Comb. Optim.},
    month = jul,
    pages = {87–96},
    numpages = {10}
}

@ARTICLE{BhadoriyaRDCM24,
    author={Abhay Singh Bhadoriya and Sivakumar Rathinam and Swaroop Darbha and David W. Casbeer and Satyanarayana G. Manyam},
    title={Assisted Path Planning for a {UGV-UAV} Team Through a Stochastic Network},
    journal={Journal of the Indian Institute of Science}, 
    year={2024},
    volume={104},
    pages={691-710},
}

@book{CorberanL14_ArcRouting,
    title   =       "Arc Routing: Problems, Methods, and Applications",
    editor  =       "Angel Corberan and Gilbert Laporte",
    publisher =     "Society for Industrial and Applied Mathematics (SIAM)", 
    address =	"Philadelphia, PA",
    year    =       2014
}

@article{AgarwalA24,
    author  =       "Saurav Agarwal and Srinivas Akella",
    title   =       "{Line Coverage} with Multiple Robots: Algorithms and Experiments",
    journal =       "IEEE Transactions on Robotics",
    volume  =       "40",
    year    =       2024,
    pages	=	"1664--1683",
    doi = {10.1109/TRO.2024.3355802},
}

@article{Guan1962_CPP,
    author    = {Meigu Guan},
    title     = {Graphic programming using odd or even points},
    journal   = {Chinese Mathematics},
    volume    = {1},
    number    = {2},
    pages     = {273--277},
    year      = {1960}
}

@article{EdmondsJ1973_CPPMatching,
    author    = {Jack Edmonds and Ellis L. Johnson},
    title     = {Matching, {Euler Tours} and the {Chinese Postman}},
    journal   = {Mathematical Programming},
    volume    = {5},
    number    = {1},
    pages     = {88--124},
    year      = {1973},
    publisher = {Springer}
}

@inproceedings{NguyenA26,
    author = "Ninh Nguyen and Srinivas Akella",
    title = "Dynamic {UGV-UAV} Cooperative Path Planning in Uncertain Environments",
    booktitle = "IEEE International Conference on Robotics and Automation",
    month = jun,
    year = "2026",
    address = "Vienna, Austria",
    eprint        = {2604.25267},
    archivePrefix = {arXiv},
    primaryClass  = {cs.RO}
}

\end{document}